\documentclass{article}

\usepackage{microtype}
\usepackage{graphicx}
\usepackage{subfigure}
\usepackage{booktabs} 

\usepackage{hyperref}

\usepackage[accepted]{icml2020}

\icmltitlerunning{Fast OSCAR and OWL Regression via Safe Screening Rules}

\begin{document}

\twocolumn[
\icmltitle{Fast OSCAR and OWL Regression via Safe Screening Rules}

\begin{icmlauthorlist}
\icmlauthor{Runxue Bao}{to}
\icmlauthor{Bin Gu}{ed}
\icmlauthor{Heng Huang}{to,ed}
\end{icmlauthorlist}

\icmlaffiliation{to}{Department of Electrical and Computer Engineering, University of Pittsburgh, Pittsburgh, USA}
\icmlaffiliation{ed}{JD Finance America Corporation}

\icmlcorrespondingauthor{Heng Huang}{heng.huang@pitt.edu}

\icmlkeywords{Feature Screening, Ordered  Weighted Sparse Regression}

\vskip 0.3in
]



\printAffiliationsAndNotice{} 

\begin{abstract}
Ordered Weighted $L_{1}$ (OWL) regularized regression is a new regression analysis for high-dimensional sparse learning. Proximal gradient methods are used as standard approaches to solve OWL regression. However, it is still a burning issue to solve OWL regression due to considerable computational cost and memory usage when the feature or sample size is large. In this paper, we propose the first safe screening rule for OWL regression by exploring the order of the primal solution with the unknown order structure via an iterative strategy, which overcomes the difficulties of tackling the non-separable regularizer. It effectively avoids the updates of the parameters whose coefficients must be zero during the learning process. More importantly, the proposed screening rule can be easily applied to standard and stochastic proximal gradient methods. Moreover, we prove that the algorithms with our screening rule are guaranteed to have identical results with the original algorithms. Experimental results on a variety of datasets show that our screening rule leads to a significant computational gain without any loss of accuracy, compared to existing competitive algorithms.
\end{abstract}

\section{Introduction}
OWL regression \cite{bogdan2013statistical,zeng2014decreasing,bogdan2015slope,figueiredo2016ordered,bao2019efficient} has emerged as a useful procedure for high-dimensional sparse regression recently, which can promote the sparsity and grouping simultaneously. Unlike group Lasso \cite{yuan2006model} and its variants, OWL regression can identify precise grouping structures of strongly correlated covariates automatically during the learning process without any prior information of feature groups. Remarkably, \cite{bu2019algorithmic} concluded that it has two good properties to achieve the minimax estimation from the estimation side without any prior knowledge of coefficients \cite{su2016slope,bellec2018slope} and controls the false discovery rate from the testing side \cite{bogdan2015slope,brzyski2019group}, which do not simultaneously exist in other models such as Lasso \cite{tibshirani1996regression} and knockoffs \cite{barber2015controlling}. Owing to its effectiveness, OWL is widely used in various kinds of applications, \emph{e.g.}, gene expression \cite{bogdan2015slope}, brain networks \cite{oswal2016representational} and neural networks training \cite{zhang2018learning}.

Although proximal gradient methods are used as standard approaches \cite{bondell2008simultaneous,bogdan2015slope} to solve OWL regression, it still suffers from high computational cost and memory usage when the feature or sample size is large in practice. The main bottleneck is the computation to update the solution in each iteration depends on all the data points. The screening technique is an easy-to-implement and promising approach for accelerating the training of sparse learning models by eliminating the features whose coefficients must be zero, which can safely avoid these useless computation during the whole training process.

\begin{table*}[t]
	\center
	\vspace*{-5pt}
	\caption{Representative safe screening algorithms. ``Type of screening'' represents the algorithm screening samples or features. ``Size'' represents the number of the hyperparameters in regularization where d is feature size. ``Fixation'' represents whether the regularization hyperparameter of each variable is fixed during the learning process. }
    \vskip 0.15in
	\begin{tabular}{c|c|c|c|c}
		\hline
		{\textbf{Problem}}  & {\textbf{Type of screening}} & {\textbf{Size}} & {\textbf{Separability}} & {\textbf{Fixation}}  \\
		\hline
		Lasso \cite{liu2014safe} & features & 1 & Separable  & Fixed    \\
		Lasso \cite{fercoq2015mind} & features & 1 & Separable  & Fixed    \\
		Sparse SVM \cite{shibagaki2016simultaneous}  & features and samples  & 3 & Separable  & Fixed  \\ 
		Sparse-group Lasso \cite{ndiaye2016gap} & features & 2  & Separable   & Fixed  \\
		Sparse SVM \cite{zhang2017scaling}  & features and samples  & 3 & Separable  & Fixed  \\
		Proximal Weighted Lasso \cite{rakotomamonjy2019screening} & features & d & Separable   & Fixed    \vspace*{0.2pt} \\
		\hline
		OWL regression (Ours) & features & d & Non-separable  & Unfixed  \\
		\hline
	\end{tabular}
	\label{table:methods}
	\vspace*{-5pt}
\end{table*}

The safe screening rules introduced by \cite{Laurent2012safe} for generalized $l_{1}$ regularized problems eliminate features whose associated coefficients are proved to be zero at the optimum. The screening rule in \cite{Laurent2012safe} is called static safe rules, which is only performed once, prior to any optimization algorithm. Relaxing the safe rules, heuristic strategies, called strong rules \cite{tibshirani2012strong}, reduce the computational cost using an active set strategy at the price of possible mistakes, which requires difficult post-processing to check for features possibly wrongly discarded. Another road to screening method is called sequential safe rules \cite{wang2013lasso,xiang2016screening}. The sequential screening rule relies on the exact dual optimal solution, which could be very time-consuming and lead to be unsafe in practice. Recently, the introduction of safe dynamic rules \cite{fercoq2015mind} has opened a promising venue by conducting safe screening not only at the beginning of the algorithm, but also during the learning process. Following \cite{fercoq2015mind} for Lasso, many dynamic screening rules relying on the duality gap are proposed in \cite{shibagaki2016simultaneous,ndiaye2016gap,rakotomamonjy2019screening,zhai2019safe} for a broad class of sparse learning problems with both good empirical and theoretical results. 

This work is concerned with algorithmic acceleration of OWL regression through safe screening rules to safely avoid useless computation whose parameters must be zero during the training process without any influence on the final learned model. We summarized several representative safe screening algorithms in Table \ref{table:methods}. It shows that existing safe screening rules have been widely used to accelerate algorithms in sparse learning by screening useless samples or features while all of them are limited to separable penalties and the fixed regularization hyperparameter of each variable, which is essential to derive the screening rules. So far there are still no safe screening rules proposed for OWL regression. This vacuum is because OWL penalty is non-separable, meaning it cannot be written as  $\Omega_{\lambda}(\beta) = \sum_{i=1}^d\lambda_{i}\omega(\beta_{i})$. Thus, all the hyperparameters for each variable in OWL penalty are unfixed until we finish the whole learning process while they are fixed in other models at the initial stage. Besides, how to derive an efficient screening rule with the numerous hyperparameters is another key point to be considered. Because of the challenges to derive screening rules for the non-separable OWL penalty with numerous unfixed hyperparameters, speeding up OWL regression by screening rules is still an open and challenging problem.

To address these challenges, in this paper, we propose a safe screening rule for the linear regression with the family of OWL regularizers based on the intermediate duality gap, which is significantly helpful for accelerating the training algorithms. As far as we know, this work is the first attempt in this direction. We effectively overcome the difficulties caused by the non-separable penalty by exploring the order of the primal solution with the unknown order structure via an iterative strategy, which leads to better understanding of the non-separable penalty for future. Specifically, in high-dimensional tasks, as the size of non-zero coefficients is much smaller than the size of features, our screening rule can effectively identify the features whose parameters must be zero in each iteration and then accelerate the original algorithms by skipping the useless updates of these parameters. Theoretically, we not only rigorously prove that our screening rule is safe for the whole training process, but also prove that our screening rule can be safely applied to existing standard iterative optimization algorithms both in the batch and stochastic setting without any loss of accuracy. The empirical performance shows the superiority of our algorithms with significant computational gain to the most popular proximal gradient methods, \emph{e.g.}, APGD (accelerated proximal gradient descent) algorithm and SPGD (stochastic proximal gradient descent with variance reduction) algorithm.

\section{Preliminary}
\subsection{OWL Regularized Regression}
We consider the linear regression with the family of OWL norms by solving the minimization problem as follows:
\begin{eqnarray} \label{primal}
    \min\limits_{\beta}  P_{\lambda}(\beta):=\frac{1}{2}\|y-X\beta\|^2_{2} +\sum_{i=1}^d\lambda_{i}|\beta|_{[i]},
\end{eqnarray}
where $X = [x_{1},x_{2},\ldots,x_{d}] \in \mathbb{R}^{n \times d}$ is the design matrix, $y \in \mathbb{R}^{d}$ is the measurement vector, $\beta$ is the unknown coefficient vector of the model, $\lambda = [\lambda_{1}, \lambda_{2}, \ldots, \lambda_{d}]$ is a non-negative regularization parameter vector of $d$ non-increasing weights and $|\beta|_{[1]} \geq |\beta|_{[2]} \geq \ldots \geq |\beta|_{[d]} $ are the ordered coefficients in absolute value. Each feature has a corresponding regularization parameter. OWL penalty (denoted as $\Omega_{\lambda}(\beta)$ henceforth) penalizes the coefficients according to their magnitude: the larger the magnitude, the larger the penalty. OWL regression has been shown to outperform conventional Lasso in many applications, particularly when $\beta$ is sparse and $d$ is larger than $n$ \cite{bogdan2015slope}. \cite{zeng2014decreasing,figueiredo2016ordered} provided theoretical analysis of the sparsity and grouping properties of OWL penalty for sparse linear regression tasks with strongly correlated features. 

Note that OWL regression is a general form of a set of sparse learning models. For example, Lasso is a special case of (\ref{primal}) if $\lambda_{1} = \lambda_{2} = \ldots = \lambda_{d}$, where $\lambda_{i}>0$. $L_{\infty}$-norm regression is a special case of (\ref{primal}) if $\lambda_{1} > 0$ and $ \lambda_{2} = \ldots = \lambda_{d} = 0$. OSCAR \cite{bondell2008simultaneous} is another special case of (\ref{primal}) if $\lambda_{i} = \alpha_{1} + \alpha_{2}(d - i)$, where $\alpha_{1}$ and $\alpha_{2}$ are non-negative parameters. 

We get the Fermat's rule of OWL regression by subdifferentials \cite{kruger2003frechet,mordukhovich2006frechet} as follows:
\begin{eqnarray}  \label{subdiff}
    X^\top(y-X\beta^{*}) \in \partial \Omega_{\lambda}(\beta^{*}), 
\end{eqnarray}
where $\beta^*$ is the optimum of the primal and $\partial \Omega_{\lambda}(\beta^{*})$ is the subdifferential of $\Omega_{\lambda}(\beta^{*})$. 

Thank \cite{Elvira2021techreport} for pointing out the error of the optimality conditions in the previous version \cite{DBLP:journals/corr/abs-2006-16433}, we introduce a new formulation to correct it here. Denote the active set and the inactive set at the optimal as $\A^*$ and $\A'^*$ respectively, we know   $\A^*$  and $\A'^*$ is a partition of $\{1, 2, \ldots, d \}$ and derive the optimality conditions of OWL regression as follows:
\begin{eqnarray} 
    X_{\A^*}^\top(y-X\beta^{*})   \in \partial \Omega_{\lambda_{\A^*}}(\beta^{*}_{\A^*}), \\
      X_{\A'^*}^\top(y-X\beta^{*})   \in \partial \Omega_{\lambda_{\A'^*}}(\beta^{*}_{\A'^*}),
\end{eqnarray}
where $\lambda_{\A^*} = [\lambda_{1},  \ldots, \lambda_{|\A^*|}]$  and $\lambda_{\A'^*} = [\lambda_{|\A^*| + 1}, \ldots, \lambda_{d}]$ is a partition of $\lambda$.

\subsection{Proximal Gradient Methods}
Proximal gradient methods are used as standard approaches to solve OSCAR and OWL regression. However, a major drawback is that it has slow convergence. Thus, accelerated proximal gradient methods are proposed to solve the optimization problems with the non-smooth penalty. Inspired by FISTA (Fast Iterative Shrinkage-Thresholding Algorithm) \cite{beck2009fast}, \cite{zhong2012efficient} proposed an APGD algorithm to solve OSCAR by efficiently addressing the proximal operator. Further, \cite{bogdan2015slope} proposed an APGD algorithm to solve the general OWL regression with the proximal operator as:
\begin{eqnarray} 
    \prox(y,\lambda) := \argmin\limits_{x \in \mathbb{R}^{d} }  \frac{1}{2}\|y-x\|^2_{2} +\sum_{i=1}^d\lambda_{i}|x|_{[i]}.
\end{eqnarray}

Nevertheless, APGD algorithm still suffers from high computational costs and memory burden when either the size of features or samples is large. Specifically, the computation of each proximal step above takes $O(d \log d)$. The computational cost of APGD algorithm for each iteration is $O(d(n+\log d))$. 

Further, as an update of each iteration in APGD algorithm depends on all the samples, each iteration of APGD algorithm can be very expensive in large-scale learning since it requires the computation of full gradients. In large-scale learning,  SPGD algorithm is proposed in \cite{xiao2014proximal} as an effective alternative, which only requires the gradients of the samples of a mini-batch size each time.

\begin{remark}
In practice, OWL regression is typically performed in the high-dimensional setting. Hence, APGD and SPGD algorithms usually suffer from high computational costs and memory burden for large feature size $d$. Thus, it is important and promising to speed up OWL regression by the screening technique for both APGD and SPGD algorithms.
\end{remark}

\section{Screening Rule}
In this section, we first provide the dual formulation of OWL regression and then derive the screening test based on the dual formulation. Next, we provide safe screening rules for OWL regression.

\subsection{Dual of OWL Regression}
In this part, we derive the dual problem of OWL regression and the screening test for OWL regression.

We consider the primal objective (\ref{primal}) of OWL regression, which is convex, non-smooth and non-separable. Following the derivation of $l_{1}$ regularized regression in appendix E of \cite{johnson2015blitz}, let $a_{i} = X^\top_{i,:}$ and $f_{i}(z_{i}) = \frac{1}{2}(y_{i}-z_{i})^2$, we can derive the dual of  OWL regression as follows:
\begin{subequations}
\begin{small}
\begin{eqnarray} 
\!\!&  &\!\! \min\limits_{\beta}  \frac{1}{2}\|y-X\beta\|^2_{2}
    +\sum_{i=1}^d\lambda_{i}|\beta|_{[i]} \label{yy} \\
\!\!& = &\!\!  \min\limits_{\beta}  \frac{1}{2} \sum_{i=1}^{n} (y_{i}-a_{i}^\top\beta)^2 +\sum_{i=1}^d\lambda_{i}|\beta|_{[i]} \nonumber \\
\!\!& = &\!\!  \min\limits_{\beta}   \sum_{i=1}^{n} f_{i}(a_{i}^\top\beta) +\sum_{i=1}^d\lambda_{i}|\beta|_{[i]} \nonumber \\
\!\!& = &\!\! \!\!\!\!\ \min\limits_{\beta}   \sum_{i=1}^{n} f^{**}_{i}(a_{i}^\top\beta) +\sum_{i=1}^d\lambda_{i}|\beta|_{[i]} \nonumber \\
\!\!& = &\!\! \min\limits_{\beta}   \sum_{i=1}^{n} \max\limits_{\theta_{i}} [(a_{i}^\top\beta)\theta_{i}-f^{*}_{i}(\theta_{i})]  +\sum_{i=1}^d\lambda_{i}|\beta|_{[i]} \nonumber \\
\!\!& = &\!\! \min\limits_{\beta}  \max\limits_{\theta} - \sum_{i=1}^{n} f^{*}_{i}(\theta_{i}) + \beta^\top X^\top\theta +\sum_{i=1}^d\lambda_{i}|\beta|_{[i]} \nonumber \\
\!\!& = &\!\! \max\limits_{\theta} - \sum_{i=1}^{n} f^{*}_{i}(\theta_{i}) + \min\limits_{\beta} \beta^\top X^\top\theta +\sum_{i=1}^d\lambda_{i}|\beta|_{[i]}
         \label{maximization} \\
\!\!& = &\!\! \max\limits_{X^\top\theta \in C_\lambda}  \sum_{i=1}^{n} - f^{*}_{i}(\theta_{i})  \label{dual} \\
\!\!& = &\!\! \max\limits_{X^\top\theta \in C_\lambda}  - \frac{1}{2}\|\theta\|^2_{2} - \theta^\top y, \label{dual2}
\end{eqnarray}
\end{small}
\end{subequations}
where $\theta$ is the solution of the dual and $\beta \in C_\lambda$ means $\sum_{j \leq i} |\beta|_{[j]} \leq \sum_{j \leq i} \lambda_j $ for all $i = 1, \ldots, d$. 

Note that $f^{*}_{i}$ is the convex conjugate of function $f_{i}$ as:
\begin{eqnarray} 
f^{*}_{i}(\theta_{i}) = \max\limits_{z_{i}} \theta_{i} z_{i}-f_{i}(z_{i}). 
\end{eqnarray}

The penultimate step to derive the dual uses the optimality condition of the following problem:  
\begin{eqnarray} \label{optimality_origin}
\min\limits_{\beta} \beta^\top X^\top\theta +\sum_{i=1}^d\lambda_{i}|\beta|_{[i]}.
\end{eqnarray}

Suppose the order of $\beta^*$ is known, the optimality conditions of (\ref{optimality_origin}) are as follows:
\begin{eqnarray} 
    -X_{\A^*}^\top\theta^*   \in \partial \Omega_{\lambda_{\A^*}}(\beta^{*}_{\A^*}),
    \label{constraint0}
    \\
      -X_{\A'^*}^\top\theta^*   \in \partial \Omega_{\lambda_{\A'^*}}(\beta^{*}_{\A'^*}),  \label{constraint}
\end{eqnarray}
where $\theta^*$ is the optimum of the dual, which can be transformed as the constraints in (\ref{dual}). Hence, we get the dual formulation of OWL regression as above.

Suppose the optimum primal and dual solutions are known, we can derive the screening condition for each variable from the optimality condition (\ref{constraint0}) and (\ref{constraint}) as: 
\begin{eqnarray} \label{condition}
    |x^\top_{i} \theta^* | < \lambda_{|\A^*|} \Rightarrow \beta^*_{i} = 0,
\end{eqnarray}
to identify the variables whose coefficient must be zero. Then, in the latter training process, we can train the model with less parameters and features while keeping the same output. However, the optimum in the left and right term of the screening condition in (\ref{condition}) are both unknown during the training process. 

Hence, the aim of our screening rule is to screen as many variables whose coefficients should be zero as possible by constructing a small and safe region for the left term of the screening condition in (\ref{condition}) with the unknown dual optimum and exploring the unknown order structure of the primal optimum for the right term of the screening condition in  (\ref{condition}). 

\subsection{Upper Bound for the Left Term}
In this part, we derive a tight upper bound for $|x^\top_{i} \theta^* |$ in (\ref{condition}) by utilizing the intermediate duality gap at each iteration during the training process. 

By the triangle inequality, we can derive the following bound as:
\begin{eqnarray} 
    |x^\top_{i} \theta^* | \leq |x^\top_{i} \theta | + \|x_{i}\| \|\theta^* - \theta\|. \label{test}
\end{eqnarray}

Note that the dual formulation $D(\theta)$ derived in (\ref{dual2}) is as follows:
\begin{eqnarray} 
    \max\limits_{\theta}  D(\theta):= - \frac{1}{2}\|\theta\|^2_{2} - \theta^\top y, \\
      s.t. \quad X^\top\theta \in C_\lambda \nonumber
\end{eqnarray}
and thus the dual $D(\theta)$ is a strongly concave function. We have the following Property \ref{property1}. 

\begin{property}
Dual $D(\theta)$ is strongly concave \emph{w.r.t.} $\theta$. Hence, we have 
\begin{eqnarray} \label{concave}
    D(\theta) \leq D(\theta^*)-\nabla D(\theta^*)^\top(\theta^* - \theta) - \frac{1}{2}\| \theta - \theta^* \|^2_{2}.
\end{eqnarray}
\label{property1}
\end{property}

Considering Property \ref{property1}, we can further bound the distance between the intermediate solution and the optimum of the dual in Corollary \ref{corollary1} based on the first-order optimality condition of constrained optimization.

\begin{corollary}
Suppose $\theta$ and $\theta^*$ are any feasible solution and the optimum of the dual respectively, we have:
\begin{eqnarray}
    \| \theta - \theta^* \| \leq \sqrt{2G(\beta, \theta)},
\end{eqnarray}
where $G(\beta, \theta) = P(\beta) - D(\theta)$ is the intermediate duality gap.\label{corollary1}
\end{corollary}
\begin{proof}
By the first-order optimality condition for strongly concave dual $D(\theta)$, we have:
\begin{eqnarray} 
    \nabla D(\theta^*)^\top(\theta^* - \theta) \geq 0.
\end{eqnarray}
Hence, based on (\ref{concave}), we have:
\begin{eqnarray} 
     \frac{1}{2}\| \theta - \theta^* \|^2_{2}  \leq D(\theta^*) - D(\theta).
\end{eqnarray}

By strong duality that $P(\beta ) \geq D(\theta^*)$, we have 
\begin{eqnarray} 
     \frac{1}{2}\| \theta - \theta^* \|^2_{2}  \leq P(\beta ) - D(\theta),
\end{eqnarray}
which completes the proof. 
\end{proof}

Hence, we can substitute $\|\theta - \theta^* \|$ in (\ref{test}) by Corollary \ref{corollary1} based on the intermediate duality gap and then derive the screening test with the upper bound for the left term  as follows:
\begin{eqnarray} 
 |x^\top_{i} \theta | + \|x_{i}\| \sqrt{2G(\beta, \theta)} < \lambda_{|\A^*|}.
\end{eqnarray}

The intermediate duality gap can be computed by $\beta$ and $\theta$. $\beta$ and $\theta$ can be easily obtained  in the original proximal gradient algorithms.

\subsection{Iterative Strategy for the Screening Rule}
The screening condition (\ref{condition}) only works when the order of the primal optimum is known in advance, which is unknown until we finish the training process in practice. To make the screening condition applicable, we design an efficient and effective iterative strategy to explore the order of the primal optimum with the unknown order structure. 

We can do screening test first as:
\begin{eqnarray} \label{screeningfirst}
    |x^\top_{i} \theta | + \|x_{i}\| \sqrt{2G(\beta, \theta)} < \lambda_{d} \Rightarrow \beta^*_{i} = 0.
\end{eqnarray}
According to the screening test above and for the following similarly, we can partition the variables into a safe active set $\mathcal{A}$ and a safe inactive set $\mathcal{A}’$ where the active set is the set of the variables that cannot be removed yet by our screening rule and the inactive set is the complementary set of the active set.

Suppose active set $\mathcal{A}$ has $m$ active features at iteration $k$, we can assign an arbitrary  permutation of $d-m$ smallest parameters $\lambda_{m+1}, \lambda_{m+2}, \ldots, \lambda_{d}$ to these screened coefficients without any influence to the final learned model. Thus, the order of these variables whose coefficients must be zero is known to be $d-m$ minimal absolute values of all. 

Then, with $m$ active features, by doing screening test as:
\begin{eqnarray} \label{screening}
    |x^\top_{i} \theta | + \|x_{i}\| \sqrt{2G(\beta, \theta)} < \lambda_{m} \Rightarrow \beta^*_{i} = 0,
\end{eqnarray}
we can find new active set $\mathcal{A}$ with $m'$ active features where $m' \leq m$ and further derive the order of the $m-m'$ screened variables by assigning the parameters similarly as above.

At each iteration, we repeat the screening test to explore the order of primal optimum until the active set keeps unchanged. The procedure of our iterative screening rule is summarized in Algorithm \ref{algorithm1}. 

The following Property \ref{property2} show our screening rule is safe to screen the variables whose coefficients should be zero with the unknown dual optimum and the unknown order structure of the primal optimum.  

\begin{property}
The iterative screening rule we proposed is guaranteed to be safe for Algorithm \ref{algorithm1} and the whole training process of OWL regression. \label{property2}
\end{property}

\begin{proof}
First, we prove our screening rule is safe for Algorithm \ref{algorithm1}. At the first iteration of Algorithm \ref{algorithm1}, active set $\mathcal{A}$ has total $d$ active features. We do screening test (\ref{screeningfirst}). Since $\lambda = [\lambda_{1}, \lambda_{2}, \ldots,\lambda_{d}]$ is a non-increasing vector, we have $\lambda_{d} \leq \lambda_{|\A^*|}$. Hence, the screening test above can make sure 
$|x^\top_{i} \theta | + \|x_{i}\| \sqrt{2G(\beta, \theta)} < \lambda_{|\A^*|}$. Thus, our screening test is safe at the first iteration.

Suppose our screening test is safe for the first $k$ iterations and active set $\mathcal{A}$ has $m$ active features at iteration $k$, the parameters of the $d-m$ screened variables whose coefficients should be zero at the optimum are assigned as a permutation of $[\lambda_{m+1}, \lambda_{m+2}, \ldots, \lambda_{d}]$. Then, the new regularization parameter vector for the variables that has not been screened is a permutation of $\lambda = [\lambda_{1}, \lambda_{2}, \ldots, \lambda_{m}]$. 

Thus, we can do the screening test for the left active variables as (\ref{screening}) to make sure $|x^\top_{i} \theta | + \|x_{i}\| \sqrt{2G(\beta, \theta)} < \lambda_{|\A^*|}$, which shows the screening test is safe at iteration $k+1$. Thus, our screening rule is proved to be safe for Algorithm \ref{algorithm1}.

For the latter sub-problem with less parameters and features to be solved in the iterative optimization algorithm, the way to do the screening test is similar to the original problem. Thus, following the proof above, we can easily prove that our screening rule is safe for the latter sub-problem and further for the whole training process of OWL regression, which completes the proof for Property \ref{property2}.

\end{proof}

\begin{algorithm}[ht]
\renewcommand{\algorithmicrequire}{\textbf{Input:}}
\renewcommand{\algorithmicensure}{\textbf{Output:}}
\caption{Safe Screening Rule for OWL Regression with Iterative Strategy}
\begin{algorithmic}[1]
\REQUIRE $\mathcal{A},\beta_{k},\theta_{k},G(\beta_{k},\theta_{k})$.
\WHILE{$\mathcal{A}$ still changes}
\STATE Do the screening test based on (\ref{screening}).
\STATE Update $\mathcal{A}$.
\ENDWHILE
\ENSURE New active set $\mathcal{A}$.
\end{algorithmic}
\label{algorithm1}
\end{algorithm}

\section{Screening Rule in the Proximal Gradient Algorithms}
In this section, we apply the screening rule to the APGD and SPGD algorithm in the batch and stochastic setting respectively for OWL regression.

\subsection{Proposed Algorithms}
In the batch setting, we compute the dual solution and duality gap first. Then, we compute the active set by Algorithm \ref{algorithm1} and update the solution as the original APGD algorithm with the obtained active variables. If active set $\mathcal{A}$ is updated in the current iteration, we also update the step size.  As the iteration increases, the solution is closer to the optimum and thus the duality gap also becomes smaller. Correspondingly, more inactive variables are screened by our screening rule. We present the procedures of our algorithm for the batch setting in Algorithm \ref{algorithm2}.

Similarly, in the stochastic setting, we compute the dual solution and duality gap in the main loop first. After that, we derive the active set by Algorithm \ref{algorithm1} and update the solution as the original SPGD algorithm with the obtained active variables. Let $F(\beta):=\frac{1}{2}\|y-X\beta\|^2_{2}$, we present the procedures of our algorithm for the stochastic setting in Algorithm \ref{algorithm3}.

\begin{algorithm}[ht]
\renewcommand{\algorithmicrequire}{\textbf{Input:}}
\renewcommand{\algorithmicensure}{\textbf{Output:}}
\caption{Accelerated Proximal Gradient Descent Algorithm  with Safe Screening Rules}
\begin{algorithmic}[1]
\REQUIRE $\beta^{0},b^{1}=\beta^{0},t_{1}=1$.
\FOR{$k=1,2,\ldots$}
\STATE Compute dual $\theta$ and duality gap.
\STATE Update $\mathcal{A}$ based on Algorithm \ref{algorithm1}.
\IF{$\mathcal{A}$ changes}
\STATE $t_{k} = t_{1}$.
\ENDIF
\STATE $\beta^{k} = \prox_{t_{k},{\lambda}}(b^{k}-t_{k}X^\top(X b^{k}-y))$.
\STATE $t_{k+1}=\frac{1}{2}(1+\sqrt{1+4 t_{k}^{2}})$.
\STATE $b^{k+1} = \beta^{k} + \frac{t_{k}-1}{t_{k+1}}(\beta^{k}-\beta^{k-1})$.
\ENDFOR
\ENSURE Coefficient $\beta$.
\end{algorithmic}
\label{algorithm2}
\end{algorithm}

\begin{algorithm}[ht]
\renewcommand{\algorithmicrequire}{\textbf{Input:}}
\renewcommand{\algorithmicensure}{\textbf{Output:}}
\caption{Stochastic Proximal Gradient Descent Algorithm with Safe Screening Rules}
\begin{algorithmic}[1]
\REQUIRE $\beta^{0},l$.
\FOR{$k=1,2,\ldots$}
\STATE Compute dual $\theta$ and duality gap.
\STATE Update $\mathcal{A}$ based on Algorithm \ref{algorithm1}.
\STATE $\beta = \beta^{k-1}$.
\STATE $\tilde{v} = \nabla F(\beta)$.
\STATE $\tilde{\beta}^{0} = \beta$.
\FOR{$t=1,2,\ldots,T$}
\STATE Pick mini-batch $I_{t} \subseteq X$ of size $l$.
\STATE $v_{t}=(\nabla F_{I_{t}}(\tilde{\beta}^{t-1}) - \nabla F_{I_{t}}(\beta))/l + \tilde{v}  $.
\STATE $\tilde{\beta}^{t} = \prox_{\eta,{\lambda}}(\tilde{\beta}^{t-1}-\eta v_{t})$.
\ENDFOR
\STATE $\beta^{k} = \tilde{\beta}^{T}$
\ENDFOR
\ENSURE Coefficient $\beta$.
\end{algorithmic}
\label{algorithm3}
\end{algorithm}

Interestingly, the duality gap, which is the main time-consuming step of our screening rule in Algorithm \ref{algorithm1}, has been computed by the original APGD and SPGD algorithms. Moreover, suppose the size of the active set for iteration $k$ is $d_{k}$, the computation complexity of the screening rule for each iteration is only $O(d_{k})$, which is even cheaper than the complexity of the original stopping criterion evaluation $O(d)$ and thus can be skipped for the analysis with the complexity $O(d_{k}(n+\log d_{k}))$ or $O(d_{k}(n+ T l + T \log d_{k}))$ for each iteration in the batch and stochastic setting respectively. 

More importantly, for iteration $k$ with $d_{k}$ active variables, our Algorithm \ref{algorithm2} only requires $O(d_{k}(n+\log d_{k}))$, which is much smaller than the complexity $O(d(n+\log d))$ required by the original APGD algorithm. Similarly, our Algorithm \ref{algorithm3} only requires $O(d_{k}(n+ T l + T \log d_{k}))$ for main loop $k$ where $T$ is number of the inner loop and $l$ is the size of mini-batch, which is much smaller than $O(d(n+T l + T \log d))$ required by the original SPGD algorithm. Hence, in high-dimensional sparse learning, the computation costs of both APGD and SPGD algorithms are effectively reduced by our screening rule.

\subsection{Theoretical Analysis}
In this part, we give the properties of convergence and screening ability when our screening rule is applied to standard iterative optimization algorithms.

In terms of the convergence, our algorithms have the following Property \ref{property3}.

\begin{property}
Suppose iterative algorithm $\Psi$ to solve OWL regression converges to the optimum, algorithm $\Psi$ with our screening rule to solve OWL regression also converges to the optimum. \label{property3}
\end{property}

\begin{proof}
We denote the sub-problem at iteration $k$ as $P_{k}$. First, we know $\Psi$ converges to the optimum for $P_{1}$. Then, suppose algorithm $\Psi$ with the screening rule converges to the optimum for $P_{k}$. Considering iteration $k+1$, $P_{k+1}$ is a sub-problem of $P_{k}$. Thus, the convergence of $P_{k+1}$ can be guaranteed as $P_{k}$, which completes the proof. 
\end{proof}
Property \ref{property3} shows the convergence of standard iterative optimization algorithms with our screening rules can be guaranteed by the original algorithms. Thus, our screening rule can be combined with existing iterative optimization algorithms, \emph{e.g.}, APGD, SPGD and \emph{et al.}.

In terms of the screening ability, our algorithms have the following Property \ref{property4} and \ref{property5}.
\begin{property}
$\theta$ converges to $\theta^*$ of the dual if $\beta$ converges to $\beta^*$ of the primal. \label{property4}
\end{property}

\begin{proof}
Considering the maximization part of (\ref{maximization}) as follows:
\begin{eqnarray} 
   \max\limits_{\theta} - \frac{1}{2}\|\theta\|^2_{2} - \theta^\top(y-X\beta),
\end{eqnarray}
we can get the primal-dual link equation as:
\begin{eqnarray} 
   \theta^* = X\beta^* - y.
\end{eqnarray}
Thus, as $\beta$ converges to $\beta^*$ of the primal, $\theta$ converges to $\theta^*$ of the dual.
\end{proof}

Property \ref{property4} shows the convergence of the dual can be guaranteed by the convergence of the primal, which means the intermediate duality gap becomes smaller as the iteration increases and thus our screening rule is promising to screen more inactive variables. 

Further, we give Property \ref{property5} to show the excellent screening ability of our screening rule. 

\begin{property}
Based on the optimality conditions, we have that final active set $ \mathcal{A}^{*}$ satisfies that  $ -X_{\A^*}^\top\theta^*   \in \partial \Omega_{\lambda_{\A^*}}(\beta^{*}_{\A^*}) $. Then, as algorithm $\Psi$ converges, there exists an iteration number $K_{0} \in \mathbb{N}$ s.t. $\forall k \geq K_{0}$, any variable $j \notin \mathcal{A}^{*}$ is screened by our screening rule. \label{property5}
\end{property}

\begin{proof}
As $\Psi$ converges, owing to the strong duality, the intermediate duality gap converges towards zero. Thus, for any given $\epsilon$, there exists $K_{0}$ such that $\forall k \geq K_{0}$, we have 
\begin{eqnarray} \label{converge1}
    \| \theta^{k} - \theta^* \|_{2} \leq  \epsilon,
\end{eqnarray}
and
\begin{eqnarray} \label{converge2}
 \sqrt{2G(\beta^{k},\theta^{k})}\leq  \epsilon. 
\end{eqnarray}

For any $j \notin \mathcal{A}^{*} $, we have 
\begin{eqnarray} 
&& |x^\top_{j} \theta^{k} | + \|x_{j}\| \sqrt{2G(\beta^{k},\theta^{k})} \\
& \leq &  |x^\top_{j} (\theta^{k} \!-\!\theta^*) | + |x^\top_{j} \theta^* | + \|x_{j}\| \sqrt{2G(\beta^{k},\theta^{k})} \nonumber \\
& \leq &  2 \|x_{j}\|\epsilon + |x^\top_{j} \theta^* | \nonumber
\end{eqnarray}

The first inequality is obtained by the triangle inequality and the second inequality is obtained by (\ref{converge1}) and (\ref{converge2}). Thus, if we choose  
\begin{eqnarray} 
    \epsilon < \frac{\lambda_{|\mathcal{A}^{*}|}-|x^\top_{j} \theta^* | }{2\|x_{j}\|} 
\end{eqnarray}
where $\lambda_{|\mathcal{A}^{*}|}-|x^\top_{j} \theta^* | > 0 $ is easily obtained since  $j \notin \mathcal{A}^{*}$, we have $|x^\top_{j} \theta^{k} | + \|x_{j}\| \sqrt{2G(\beta^{k}, \theta^{k})} < \lambda_{|\mathcal{A}^{*}|} $, which is the screening rule we proposed. That is to say, variable $j$ is screened out by our screening rule at this iteration, which completes the proof.
\end{proof}

Property \ref{property5} shows all the inactive variables $j \notin \mathcal{A}^{*}$ are correctly detected and effectively screened by our screening rule in a finite number of iterations.

\section{Experiments}
In this section, we first give the experimental setup and then present our experimental results with discussions.

\subsection{Experimental Setup}
\subsubsection{Design of Experiments}
We conduct experiments on six real-world benchmark datasets not only to verify the effectiveness of our algorithm on reducing running time, but also to show the effectiveness and safety on screening inactive variables.

To validate the effectiveness of our algorithms on reducing running time, we evaluate the running time of our algorithms and other competitive algorithms to solve OWL regression under different settings. To confirm the effectiveness and safety of our algorithms on screening inactive variables, we evaluate the screening rate at each iteration of our algorithm and the prediction errors of different algorithms. The compared algorithms are summarized as follows:
\begin{itemize}[leftmargin=0.2in]
  \item APGD: Accelerated proximal gradient descent algorithm \cite{bogdan2015slope}.
  \item APGD + Screening: Accelerated proximal gradient descent algorithm with the safe screening rule.
  \item SPGD: Stochastic proximal gradient descent algorithm with variance reduction we adopt in \cite{xiao2014proximal}. 
  \item SPGD + Screening: Stochastic proximal gradient descent algorithm with variance reduction and the safe screening rule.
\end{itemize}

\begin{table}[t]
\setlength{\tabcolsep}{1pt} 
\caption{The real-world datasets used in the experiments.}
\vskip 0.15in
\begin{center}
\begin{sc}
\begin{tabular}{lcc}
\toprule
Dataset & Sample size & Attributes \\
\midrule

    Duke Breast Cancer & 44  & 7129    \\

    Colon Cancer    & 62 & 2000      \\

    Cardiac Left    & 3360       & 1600  \\

    Cardiac Right &  3360       & 1600  \\

    IndoorLoc Longitude     & 21048 & 529      \\

    Slice Localization & 53500 & 386      \\

\bottomrule
\end{tabular}
\end{sc}
\end{center}
	\label{table:datasets}
\end{table}

\begin{table*}[t]
\footnotesize
\vspace{-6pt}
\setlength{\tabcolsep}{1pt} 
\caption{Prediction errors of different algorithms.}
\vskip 0.15in
\begin{center}
\begin{sc}
\setlength{\tabcolsep}{4.9mm}{
\begin{tabular}{lcccc}
\toprule
Dataset & APGD & APGD + Screening & SPGD & SPGD + Screening \\
\midrule

 Duke Breast Cancer & 0.6523  & \textbf{0.6523} & 0.6523  & \textbf{0.6523}    \\

 Colon Cancer     & 0.9453 & \textbf{0.9453} & 0.9453 & \textbf{0.9453}      \\

  Cardiac Left    & 0.4756 & \textbf{0.4756}    & 0.4756 &\textbf{0.4756}   \\

Cardiac Right &  0.5276 & \textbf{0.5276} &  0.5276      &  \textbf{0.5276}  \\

IndoorLoc Longitude    & 0.5531 &\textbf{0.5531} & 0.5531 & \textbf{0.5531}      \\

Slice Localization & 0.6162 &\textbf{0.6162} & 0.6162 & \textbf{0.6162}      \\

\bottomrule
\end{tabular}}
\end{sc}
\end{center}
	\label{table:errors}
\end{table*}

\begin{figure*}[!t]
  \centering
  \subfigure[Duke Breast Cancer]{
    \includegraphics[height=3.8cm]{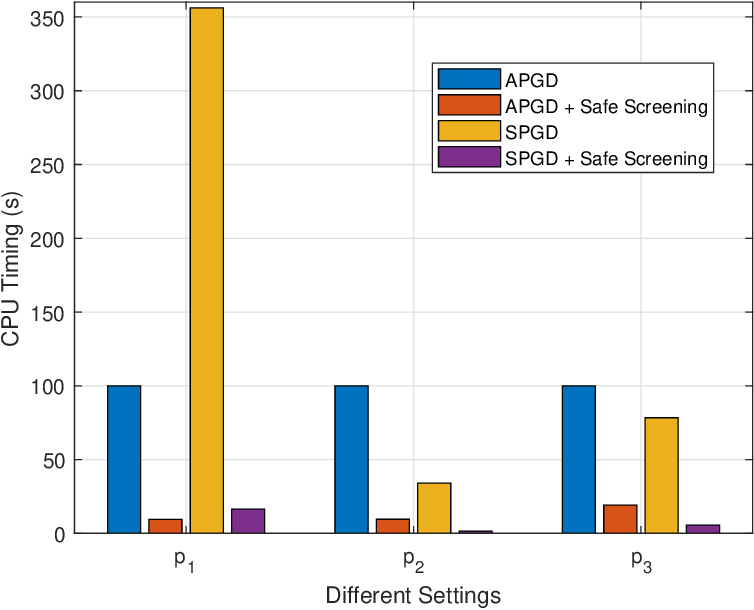}}
\qquad  \subfigure[Colon Cancer]{
    \includegraphics[height=3.8cm]{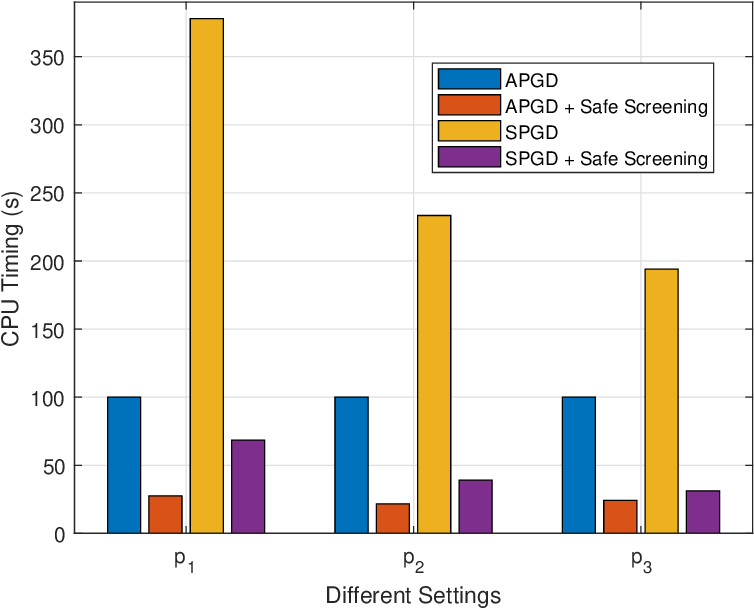}}
\qquad    \subfigure[Cardiac Left ]{
    \includegraphics[height=3.8cm]{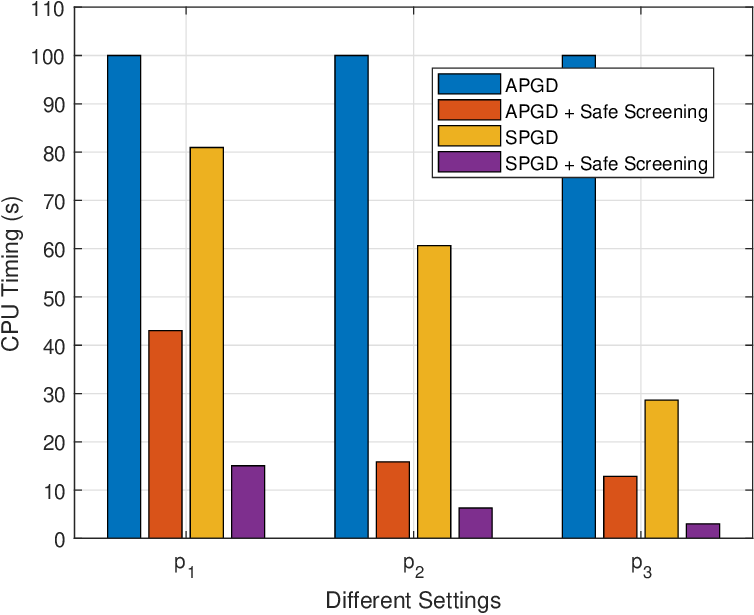}}
     \subfigure[Cardiac Right]{
    \includegraphics[height=3.8cm]{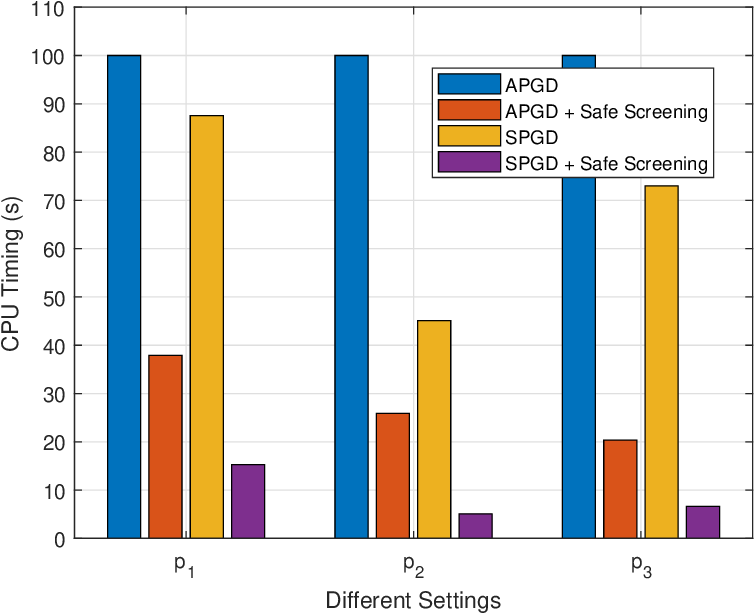}}
\qquad       \subfigure[IndoorLoc Longitude]{
    \includegraphics[height=3.8cm]{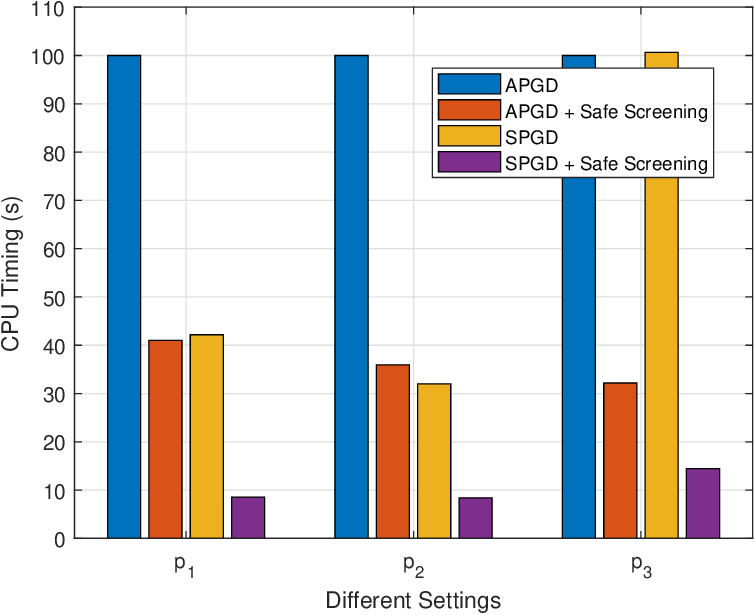}}
\qquad       \subfigure[Slice Localization ]{
    \includegraphics[height=3.8cm]{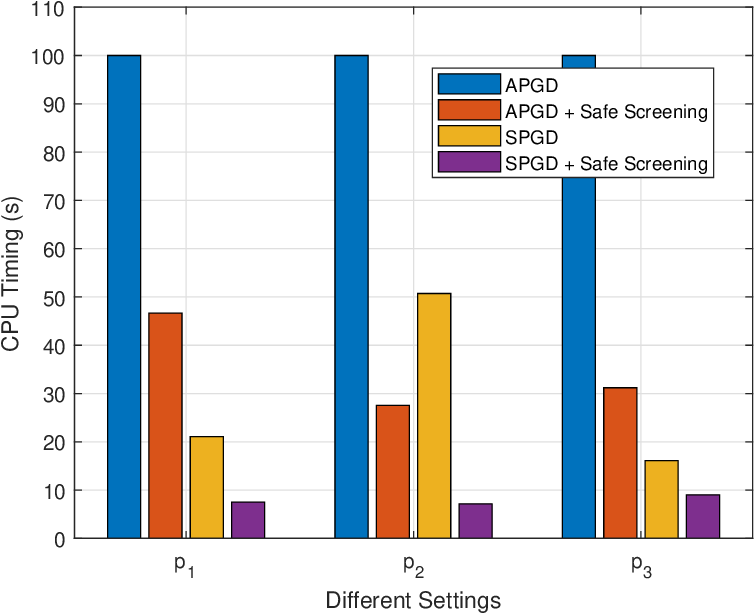}}
  \caption{Average running time of different algorithms without and with safe screening rules under different settings.}
\label{fig1} 
\end{figure*}

\begin{figure*}[!t]
  \centering
  \subfigure[Duke Breast Cancer]{
    \includegraphics[height=3.8cm]{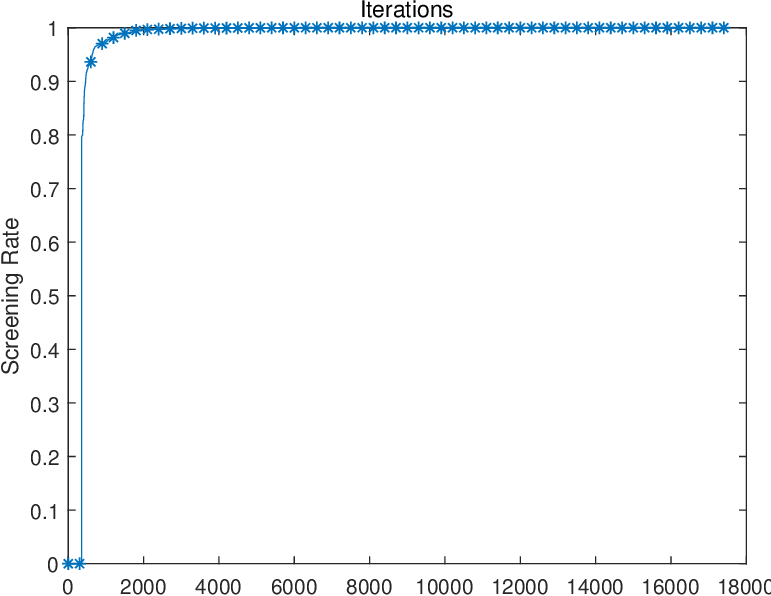}}
\qquad  \subfigure[Colon Cancer]{
    \includegraphics[height=3.8cm]{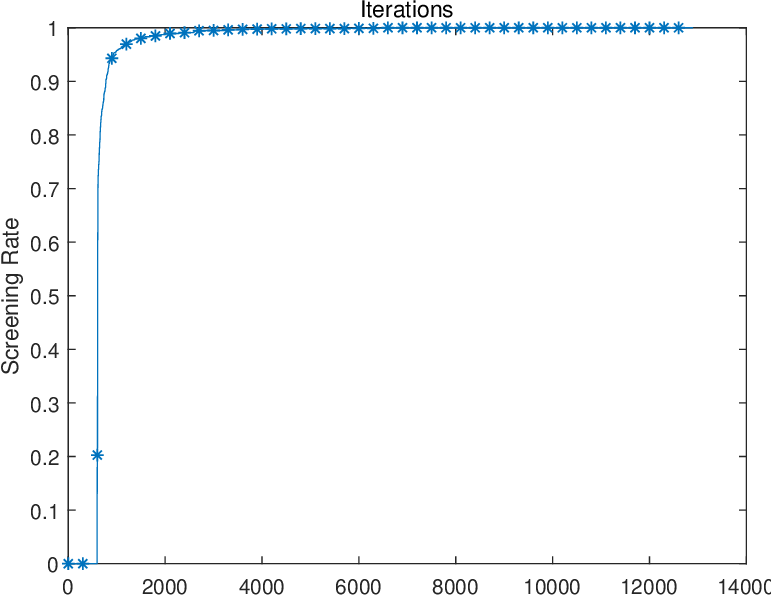}}
\qquad  \subfigure[Cardiac Left]{
    \includegraphics[height=3.8cm]{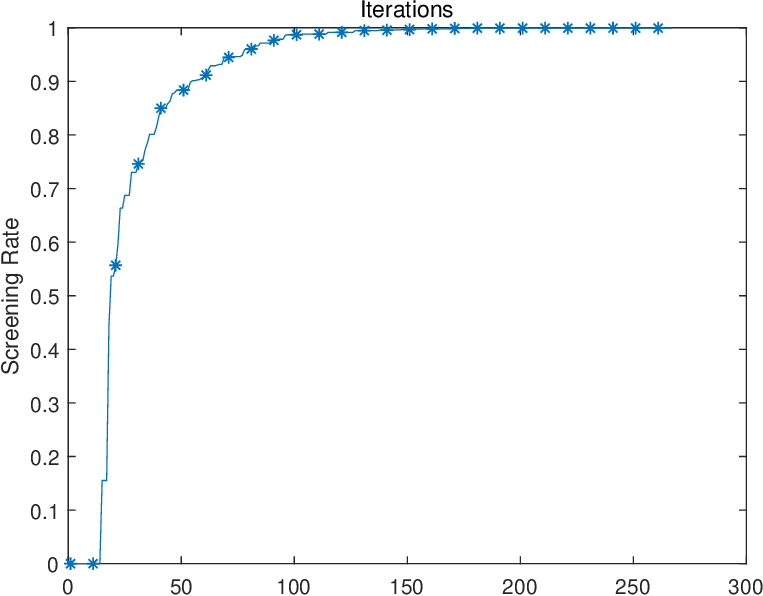}}
     \subfigure[Cardiac Right]{
    \includegraphics[height=3.8cm]{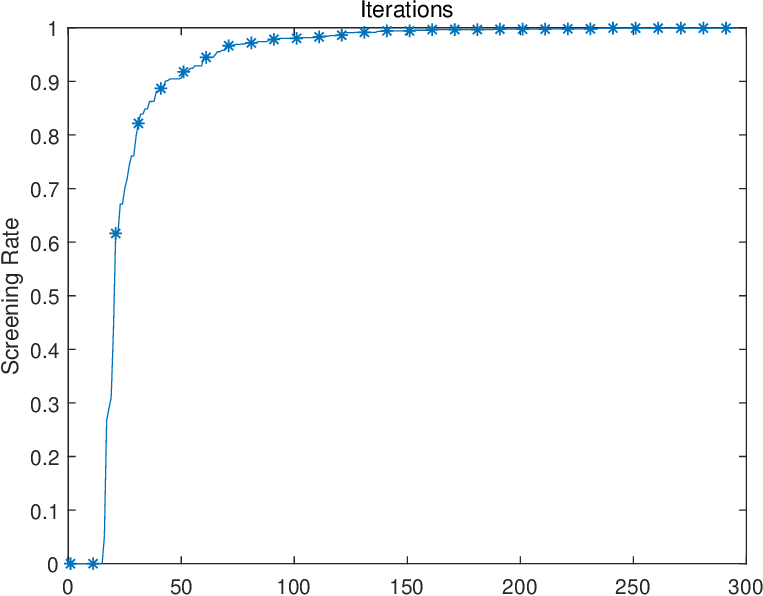}}
\qquad     \subfigure[IndoorLoc Longitude]{
    \includegraphics[height=3.8cm]{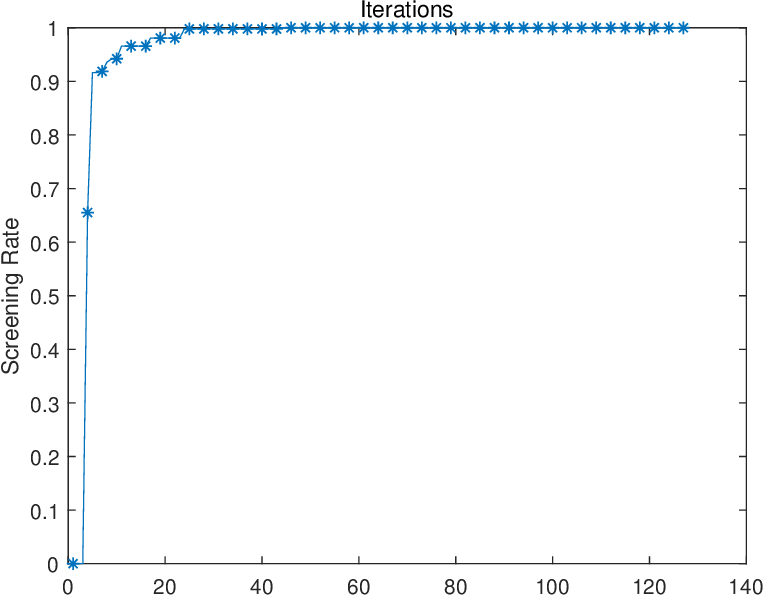}}
\qquad     \subfigure[Slice Localization]{
    \includegraphics[height=3.8cm]{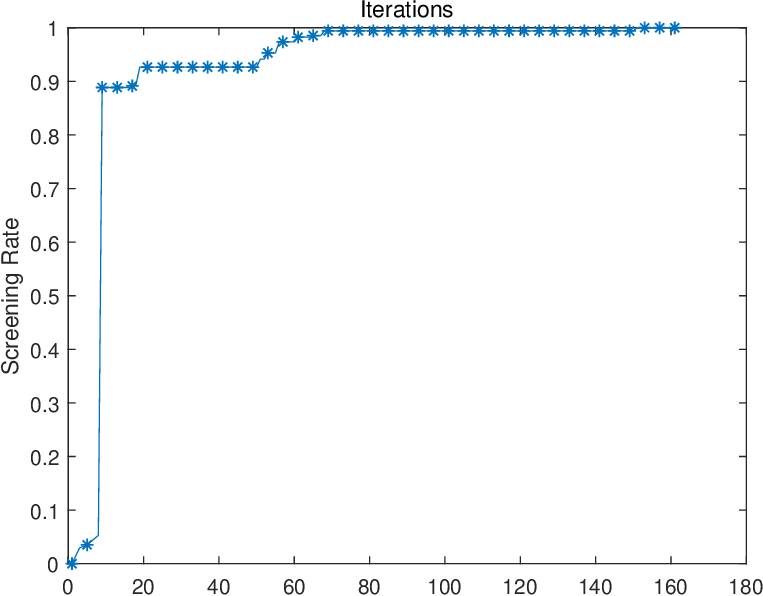}}
  \caption{The screening rate of different datasets in the stochastic setting.}
\label{fig2} 
\end{figure*}

\subsubsection{Implementation Details}
Our experiments were performed on a 4-core Intel i7-6820 machine. We implement all the algorithms in MATLAB and compare the average running CPU time of different algorithms at the same platform for 5 trials. For the comparison convenience, the CPU time of each algorithm is shown as the percentage of APGD under each setting. Following the setting in \cite{bogdan2015slope}, tolerance error $\epsilon$ of duality gap and dual infeasibility in our experiments are set as $10^{-6}$. At the very early stage, the solution is far from the optimum and thus the screening rule can only screen a small portion of variables. We run our algorithms with a warm start. Please note all the experimental setup in Algorithm \ref{algorithm2} and \ref{algorithm3} follows the original APGD and SPGD algorithms with the same hyperparameters of the size of mini-batch, the number of inner loop and step size $\eta$, which range from 5 to 100, 5 to 80 and $10^{-6}$ to $10^{-3}$ respectively for different datasets, are selected by grid search.  

We use the popular OSCAR setting (also called OWL regression with linear decay), which is widely used in \cite{oswal2016representational, zhong2012efficient, zhang2018learning}, as follows:
\begin{eqnarray} 
    \lambda_{i} = \alpha_{1} + \alpha_{2}(d-i),
\end{eqnarray}
where $\alpha_{1} = p_{i} \|X^\top y\|_{\infty}$ and $\alpha_{2} = \alpha_{1}/d$. For a fair comparison, the factor $p_{i}$ is used to control the sparsity. In our experiments, we set $p_{i} = i * e^{-\tau}$, $i = 1,2, 3$, $\tau =2 $ for Duke Breast Cancer, IndoorLoc Longitude and Slice Localization datasets and $\tau =3 $ for Colon Cancer, Cardiac Left and Cardiac Right datasets. 

To evaluate the screening rate of our algorithms, the screening rate is defined as the percentage of the inactive variables we screened to the total inactive ones. We set the sparsity as $p_{1}$ here and for the following part.

To compare the prediction error of different algorithms, we randomly divide the dataset into the training and testing set in proportion to $4:1$ and use root mean squared error (RMSE) as the performance criterion of the linear regression tasks.  

\subsubsection{Datasets}
Table \ref{table:datasets} summarizes six benchmark datasets used in our experiments. Duke Breast Cancer and Colon Cancer datasets are from the LIBSVM repository, which is available at \url{https://www.csie.ntu.edu.tw/~cjlin/libsvmtools/datasets/}. IndoorLoc Longitude and Slice Localization datasets are from the UCI benchmark repository \cite{Dua:2019}, which is available at \url{https://archive.ics.uci.edu/ml/datasets.php}. Cardiac Left and Cardiac Right datasets are collected from 3360 MRI images by hospitals \cite{gu2014incremental}.

\subsection{Experimental Results and Discussions}
\subsubsection{Running Time}
Figures \ref{fig1}(a)-(f) provide the results of the average running time of four algorithms on the six datasets for the OWL regularized regression tasks in different situations. The results confirm that the methods with our screening rule are always much faster than the original ones both in the batch and stochastic settings. This is because our screening rule could screen a large portion of inactive variables during the training process. Thus, the algorithms with our screening rule reduce much computational cost of the original algorithms.

When $n \ll d$, the results show, with our safe screening rule, APGD algorithm achieves the computational gain to the original algorithm by a factor of 4x to 8x and SPGD algorithm achieves the computational gain to the original one by 5x to 22x. For large-scale learning where $n \approx d$ or $n \gg d$, the results show SPGD algorithm with our safe screening rule always achieve the largest computational gain, which can accelerate the original APGD algorithm by 4x to 40x. This is because the stochastic methods can reduce computational burden in large-scale learning. Interestingly, with our screening rule, stochastic methods could achieve significant computational gain even when $n \approx d$. This is because the problem degenerates into a sub-problem that $n \gg d$ by screening inactive variables during the training process.  Also note we benefit from the screening rule more with larger and sparser datasets.

\subsubsection{Screening Rate}
Figures \ref{fig2}(a)-(f) present the results of the screening rate of our algorithms on six datasets in the stochastic setting to show the screening ability and characteristics of our screening rule. The results support the conclusion that our algorithm can successfully screen most of the inactive variables at the very early stage, reach the final active set and screen almost all the inactive variables in a finite number of iterations and thus is an effective method to screen inactive variables of OWL regression. This is because the upper bound of our screening test is very tight and the iterative strategy is effective to explore the order structure of primal solution to screen more inactive variables during the training process.

\subsubsection{Prediction Error}
Table \ref{table:errors} provides the results of prediction errors of four algorithms on six datasets for OWL regularized regression to confirm the safety of our screening rule.  According to the experimental results, the prediction errors of our algorithms are identical with the original algorithms. The reason is that our screening rule is guaranteed to be safe and thus our algorithms with our screening rule are guaranteed to yield the exactly same solution as the original ones.

\section{Conclusion}
In this paper, we propose the first safe screening rule for OWL regression by effectively tackling the non-separable penalty, which allows to avoid the useless computation of the parameters whose coefficients must be zero. Moreover, the proposed screening rule can be easily applied to existing iterative optimization algorithms. Theoretically, we prove that the algorithms with our screening rule is able to guarantee identical results with the original algorithms. Extensive experiments on six benchmark datasets verify that the screening rule leads to significant computation gain without any loss of accuracy by screening inactive variables.

\bibliography{123}
\bibliographystyle{icml2020}

\end{document}